\documentclass[12pt]{article}
\usepackage[utf8]{inputenc} 
\usepackage[T1]{fontenc}    
\usepackage[colorlinks]{hyperref}
\usepackage{url}            
\usepackage{booktabs}       
\usepackage{amsfonts}       
\usepackage{nicefrac}       
\usepackage{microtype}      
\usepackage{xcolor}         
\usepackage{comment}
\usepackage{amsmath,amsthm,amssymb,multirow,csquotes,cleveref}
\usepackage{xhfill,tikz-network}
\usepackage{wrapfig}
\usepackage[export]{adjustbox}
\usepackage{graphicx,environ}
\usepackage{tikz}
\usepackage{tikzscale}
\usepackage{algorithm}
\usepackage{algorithmic}
\usepackage{pgfplots,enumitem,color,colortbl}
\usepackage{setspace}
\pgfplotsset{compat=newest}
\usetikzlibrary{automata,arrows,calc,positioning}

\newcommand{\norm}[1]{\left\| #1 \right\|}
\renewcommand{\Pr}{\mathrm{Pr}}

\renewcommand{\r}{\mathbf{r}}
\renewcommand{\S}{\mathcal{S}}

\newcommand{\N}{\mathcal{N}}
\renewcommand{\a}{\mathbf{a}}

\DeclareMathOperator{\ex}{\mathbb{E}}

\newtheorem{theorem}{Theorem}[section]

\newtheorem{lemma}[theorem]{Lemma}
\newtheorem{proposition}[theorem]{Proposition}

\theoremstyle{definition}
\newtheorem{definition}{Definition}
\newtheorem{example}{Example}
\theoremstyle{remark}
\newtheorem{remark}{Remark}

\title{Independent Natural Policy Gradient Always Converges in Markov Potential Games}

\usepackage{authblk}

\author{Roy Fox, Stephen McAleer, Will Overman, Ioannis Panageas}

\affil{University of California, Irvine}
\begin{document}
\date{}
\maketitle

\begin{abstract}
  Multi-agent reinforcement learning has been successfully applied to fully-cooperative and fully-competitive environments, but little is currently known about mixed cooperative/competitive environments. 
  In this paper, we focus on a particular class of multi-agent mixed cooperative/competitive stochastic games called Markov Potential Games (MPGs), which include cooperative games as a special case. 
  Recent results have shown that independent policy gradient converges in MPGs but it was not known whether Independent Natural Policy Gradient converges in MPGs as well.  
  We prove that Independent Natural Policy Gradient always converges in the last iterate using constant learning rates. The proof deviates from the existing approaches and the main challenge lies in the fact that Markov Potential Games do not have unique optimal values (as single-agent settings exhibit) so different initializations can lead to different limit point values. We complement our theoretical results with experiments that indicate that Natural Policy Gradient outperforms Policy Gradient in routing games and congestion games. 
\end{abstract}

\section{Introduction}
Reinforcement learning has proven successful on a number of sequential decision making tasks \cite{Mni15, akkaya2019solving, arulkumaran2017deep}. However, much of this recent progress has been in single-agent domains. In multi-agent reinforcement learning domains, many of the assumptions that are useful for understanding single agent reinforcement learning do not hold. For example, from the perspective of a single agent, the MDP is not stationary over time as other agents change their policies. The complications arising in the multi-agent reinforcement learning setting make it much less understood than the single-agent setting. 

Despite the difficulty of multi-agent reinforcement learning (RL), recent RL algorithms have been developed for the fully-competitive two-player setting~\cite{lanctot2017unified, heinrich2016deep, brown2019deep, mcaleer2020pipeline, mcaleer2021xdo, hennes2019neural, daskalakis2021independent}, achieving superhuman performance on StarCraft \cite{vinyals2019grandmaster}, Go and Chess \cite{silver2017mastering}, and Poker \cite{brown2018superhuman, moravvcik2017deepstack}. 
Recent multi-agent RL algorithms have also achieved success on the multi-agent fully cooperative reward setting \cite{foerster2017learning, lowe2017multi, sunehag2017value, foerster2018counterfactual, oroojlooyjadid2019review, kuba2021trust}. Very recently, Markov Potential Games (MPGs) have emerged as a generalization of fully cooperative settings \cite{leonardos21, zhang2021gradient, mguni2021learning}. A Markov game is a MPG if there exists a potential function such that if one agent changes their policy, the difference in their value function is the same as the difference in the potential function in all states. MPGs generalize normal-form potential games to multi-step Markov games and include many important examples such as routing games (described below) and cooperative games. 

In multi-agent games, the most common solution concept is a Nash equilibrium policy \cite{nash1950equilibrium}. All agents are said to be in a Nash policy if no agent can improve their reward by deviating to another policy provided that the other agents do not deviate. Since Nash equilibria exist in every finite game, Nash equilibria serve as a desriptive solution concept, because if rational agents are not in a Nash equilibrium, they will change their policy. In two-player zero sum games and cooperative games, Nash equilibria also serve as a prescriptive solution concept, and are in a sense optimal. 

Calculating Nash policies in MPGs is of interest for a variety of reasons. First, Nash policy outcomes can be compared with max-welfare outcomes to determine the price of anarchy, a measure of the inefficiency of individual selfish behavior \cite{roughgarden2009intrinsic}. Second, MPGs have the property that pure Nash policies always exist, and can be found by \emph{locally} maximizing the potential function. Therefore, MPGs allow easier analysis of equilibrium behaviour than is known for generic general-sum Markov Games. 

A natural question to ask in this setting is whether independent learners can converge to a Nash policy in the last iterate. 
The only existing algorithm guaranteed to converge in MPGs in the last iterate is independent policy gradient \cite{leonardos21, zhang2021gradient}. Although independent policy gradient will converge to a Nash policy, it can converge very slowly. This is a common problem in single-agent policy gradient (PG) as well, where policies can often be on the boundary of the action-distribution simplex, giving a small gradient signal. 

Conversely, Natural Policy Gradient (NPG) has been shown to converge much faster than PG in the single-agent setting because NPG is able to modify the gradient to induce a large update even when the gradient is very small \cite{Aga20}. The main result of this paper is to show that when $n$ agents independently run Natural Policy Gradient in a Markov Potential Game they will convergence asymptotically to a Nash equilibrium. We also show experimentally that Independent Natural Policy Gradient indeed converges faster than Independent Policy Gradient in two separate settings.
\paragraph{Our results.}

Our main technical result is summarized by the following informal statement of Theorem~\ref{thm:mainconvergence}:
\begin{theorem}[Informal]
Consider a Markov Potential Game in which all agents are updating their policies according to Independent Natural Policy Gradient. For small enough stepsize $\eta$, Independent Natural Policy Gradient exhibits last-iterate (asymptotic) convergence to Nash-equilibrium policies.
\end{theorem}

\begin{remark}[Stepsize needs to be small] The stepsize cannot be arbitrarily large, as it has been shown that the Multiplicative Weights Update Algorithm in potential games (non-sequential MPGs) exhibits chaotic behavior for large stepsizes \cite{PPP17}. Please see Remark \ref{rem:fixedpoints} that argues that Multiplicative Weights Update is effectively Natural Policy Gradient on the policy space.
\end{remark}

\paragraph{Technical Overview.} Our main technical result is to show that Indepedent Natural Policy Gradient (INPG) globally converges to equilibrium policies for a fixed stepsize in Markov Potential Games. The proof can be brokwn down into three main steps. The first step is to show that INPG is equivalent to Natural Gradient Ascent on the potential function $\Phi$ induced by the Markov Potential Game. Having shown this fact, in the second step we show that $\Phi$ is non-decreasing in each INPG iteration. This step is rather technical as it involves showing that the gradient of $\Phi$ is locally Lipschitz in the softmax parametrization, with respect to a
Mahalanobis norm\footnote{For a positive definite matrix $A$, the Mahalanobis norm is $\norm{x}_A = \sqrt{x^{\top}Ax}.$} with a positive diagonal conditioning matrix induced by the Fisher information matrix. The last step is to show that the limit points of Independent Natural Policy Gradient are indeed equilibrium policies. Note that in the softmax parametrization this does not follow immediately from the monotonic non-decrease of the potential. We prove this step by exploiting the structure of the Multiplicative Weights Update Algorithm, which is essentially the same as Natural Policy Gradient on the policy space (see \cite{Aga20} and Remark \ref{rem:fixedpoints}).

\section{Preliminaries}\label{prelims}

\paragraph{Markov Decision Process (MDP).} The notation we use is standard and largely follows \cite{Aga20} and \cite{leonardos21}. We consider a setting with $n$ agents who select actions in a  Markov Decision Process (MDP). A MDP is a tuple $(\mathcal S, N, \{\mathcal A_i,r_i\}_{i \in N}, P, \gamma, \mu)$ in which: 
\begin{itemize}[leftmargin=*, itemsep=0cm]
    \item $\mathcal S$ is a finite state space. We write $\Delta(\mathcal S)$ to denote the set of all probability distributions over the set $\mathcal S$ and $|\mathcal S|$ for its cardinality.
    \item $N=\{1,2,\dots,n\}$ is the set of agents in the game.
    \item $\mathcal A_i$ denotes a finite set of actions for agent $i$. Using common conventions in Game Theory, we write $\mathcal A=\times_{i\in N}\mathcal A_i$ and $\mathcal A_{-i}=\times_{j\neq i}\mathcal A_j$ to denote the joint action space of all agents and of all agents but $i$, respectively. We also have that $\a = (a_i,\mathbf{a_{-i}}) \in \mathcal A$ in which $a_i \in \mathcal A_i$ and $\a_{-i}\in \mathcal A_{-i}$.
    \item $r_i: \mathcal S \times \mathcal A \to [0,1]$ is the individual reward function of agent $i\in N$, i.e., $r_i(s,a_i,\a_{-i})$ is the instantaneous reward of agent $i$ when agent $i$ takes action $a_i$ and the other agents choose $\a_{-i}$ at state $s\in \mathcal S$.
    \item $P$ is the transition probability matrix, for which $P(s'\mid s,\mathbf{a})$ is the probability of transitioning from $s$ to $s'$ given that agents choose joint action $\mathbf{a} \in \mathcal A$.
    \item $\gamma \in [0,1)$ is a discount factor, same for all agents. For finite-horizon MDPs we take $\gamma = 1$.
    \item $\mu\in \Delta(\mathcal S)$ is the distribution of the initial state.
    \end{itemize}

\paragraph{Policies and Value Functions.}  For each agent \mbox{$i\in N$}, a deterministic, stationary policy $\pi_i: \mathcal S \to \mathcal A_i$ specifies the action of agent $i$ at each state $s\in \mathcal S$. A stochastic policy $\pi_i: \mathcal S \to \Delta(\mathcal A_i)$ denotes a probability distribution over the actions of agent $i$ for each state $s\in \mathcal S$. We denote by $\Delta(\mathcal A_i)^{|\mathcal S|}$, $\Delta(\mathcal A)^{|\mathcal S|}$, and $\Delta(\mathcal A_{-i})^{|\mathcal S|}$ the set of all stochastic policies for agent $i$, the set of all joint (product distribution) stochastic policies for all agents, and for all agents but $i$, respectively. A joint policy $\pi$ induces a distribution over trajectories $\tau=(s_t,\a_t,\r_t)_{t\ge0}$, where $s_0$ is drawn from the initial state distribution $\mu$,  $a_{i,t}$ is drawn from $\pi_i(\cdot\mid s_t)$ for all agents $i\in N$ and $\r_t$ depends on $s_t,\a_t$.\par

In the multi-agent setting, each agent is aiming to maximize her respective value function. The value function $V_i^{\pi}(s): \Delta(\mathcal A)^{|\mathcal S|}\to\mathbb R$ gives the expected reward of agent $i\in N$ when the MDP starts from $s_0=s$ and the agents use joint policy $\pi$
\begin{equation}\label{eq:value_function}
V^{\pi}_i(s) := \ex \left[\left.\sum_{t=0}^\infty \gamma^t r_{i}(s_t,\a_t)\right| s_0=s\right].  
\end{equation}
We also denote $V_i^{\pi}(\mu) = \ex_{s \sim \mu}\left[V^{\pi}_i(s)\right]$ if the initial state follows distribution is $\mu.$ Analogously, one can define for each agent $i$ the $Q$-function \mbox{$Q_i^{\pi} (s,\mathbf{a}) = r_{i}(s,\mathbf{a})+\gamma \mathbb{E}_{s'\sim P(.|s,\a)}[V_i^{\pi}(s')]$} and advantage function $A_i^{\pi}(s,\mathbf{a}) = Q_i(s,\mathbf{a})-V_i^{\pi}(s)$. 
The solution concept that we will be focusing on in this paper is the Nash equilibrium joint policy. 

\begin{definition}[Equilibrium Joint Policy]\label{def:nashpolicy}
A joint policy, $\pi^*=(\pi_i^*)_{i\in N}\in \Delta(\mathcal A)^{|\mathcal S|}$, is an equilibrium policy if for each agent $i$,
\[V_i^{(\pi_i^*,\pi_{-i}^*)}(s)\ge V_i^{(\pi_i,\pi_{-i}^*)}(s),\] for all  $\pi_i\in \Delta(\mathcal A_i)^{|\mathcal S|}$, and all $s\in \mathcal S$, 
i.e., if agent $i$'s policy $\pi_i^*$ maximizes her value function for each starting state $s\in \mathcal S$ given the joint policy of the other agents $\pi^*_{-i},$ and this is true for all agents $i$. Similarly, a joint policy $\pi^*=(\pi_i^*)_{i\in N}$ is an $\epsilon$-equilibrium policy if there exists an $\epsilon>0$ so that for each agent $i$
\[V_i^{(\pi_i^*,\pi_{-i}^*)}(s)\ge V_i^{(\pi_i,\pi_{-i}^*)}(s)-\epsilon,\] for all $\pi_i\in \Delta(\mathcal A_i)^{|\mathcal S|}$ and all  $s\in \mathcal S$.
\end{definition}

\paragraph{Markov Potential Games.}
For the rest of the paper we will focus on a particular class of multi-agent MDPs called Markov Potential Games.

\begin{definition}[Markov Potential Game \cite{mgu20,mguni2021learning,leonardos21, zhang2021gradient}]\label{def:potential}
A Markov Decision Process is called a \emph{Markov Potential Game (MPG)} if there exists a (state-dependent) function $\Phi^{\pi}(s):\Delta(\mathcal A)^{|\mathcal S|} \to \mathbb{R}$ so that 
\begin{equation}\label{eq:potential}
    \Phi^{(\pi_i, \pi_{-i})}(s) - \Phi^{(\pi_i',\pi_{-i})}(s) = V_i^{(\pi_i, \pi_{-i})}(s) - V_i^{(\pi_i', \pi_{-i})}(s),
\end{equation}
for all agents $i\in N$, all states $s\in \mathcal S$ and all policies $\pi_i,\pi_i'\in \Delta(\mathcal A_i)^{|\mathcal S|}$ and $\pi_{-i}\in \Delta(A_{-i})^{|S|}$. By linearity of expectation, the same is true if $s\sim \mu$ 
\[\Phi^{(\pi_i, \pi_{-i})}(\mu) - \Phi^{(\pi_i',\pi_{-i})}(\mu) = V_i^{(\pi_i, \pi_{-i})}(\mu) - V_i^{(\pi_i', \pi_{-i})}(\mu),\] where $\Phi^{\pi}(\mu) := \ex_{s \sim \mu}\left[\Phi^{\pi}(s)\right].$
\end{definition}

We conclude this subsection with the definition of the mismatch coefficient.
\paragraph{Discounted State Distribution.} It is useful to define the discounted state visitation distribution $d_{s_0}^\pi(s)$ for $s\in \mathcal S$ that is induced by a policy $\pi$ 
\begin{equation}\label{eq:visitation}
    d_{s_0}^\pi(s):=(1-\gamma)\sum_{t=0}^{\infty}\gamma^t \Pr^\pi(s_t=s\mid s_0), \quad \text{ for all } s\in \S.
\end{equation}
We write $d_\mu^\pi(s)=\ex_{s_0\sim\mu}[d_{s_0}^\pi(s)]$ to denote the discounted state visitation distribution when the initial state distribution is $\mu$.

\begin{definition}[Distribution Mismatch coefficient \cite{Aga20}]
Let $\mu$ be any distribution in $\Delta(\S)$ and let $\mathcal{O}$ be the set of policies $\pi \in \Delta(\mathcal{A})^\S$. We call 
\[M :=\max_{\pi,\tilde{\pi} \in \mathcal{O}}\norm{\frac{d^{\pi}_{\mu}}{d^{\tilde{\pi}}_{\mu}}}_{\infty}\] 
the \emph{distribution mismatch coefficient}, where $d^{\pi}_{\mu}$ is the discounted state distribution (\ref{eq:visitation}). 
\end{definition}

\begin{figure*}
  \includegraphics[width=\textwidth]{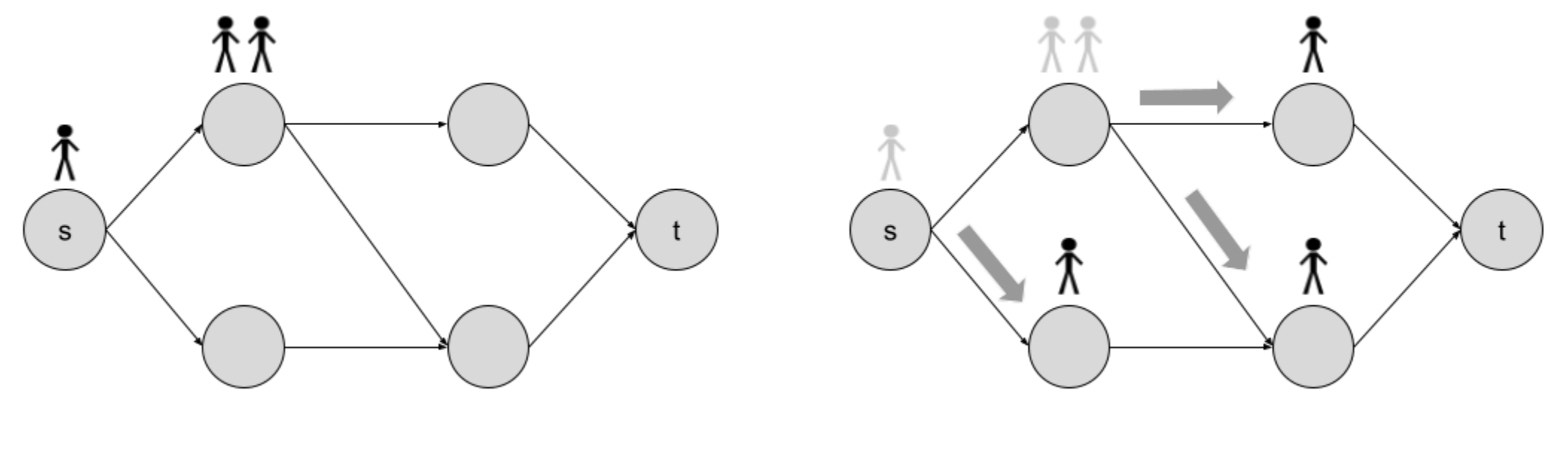}
  \caption{Two different configurations of the stochastic congestion game on 6 vertices. These correspond to two different states in the associated MDP. Arrows in the right figure show the actions chosen by the 3 agents to move from the MDP state on the left to MDP state on the right.}
  \label{fig:scg}
\end{figure*}

\begin{example}[\textbf{Stochastic Congestion Games}]\label{ex:scg}
 We are given a directed acyclic graph $G$, with a source $s$ and a destination node $t$ and $N$ agents (\Cref{fig:scg}). Every state of the finite-horizon MDP consists of the positions of all agents on the graph. The MDP terminates when they all reach $t$. Each agent chooses an action (an edge $e$ which is the node to go next) and the cost he experiences is a function of the load $l_e$ of the edge $e$ that he chooses $c_e(l_e)$, where the load is the total number of agents choosing the same edge. This setting is a MPG with the following potential function. Suppose we are at state $s$ and let $\pi := (\pi_1,...,\pi_N)$, then we have

\[\Phi^{\pi}(s) = \sum_{e \text{ reachable from }s}\mathbb{E}_{\pi}\left[ \sum_{k=1}^{l_e} c_e(k)\right].\]

Consider any player $i$. With $l_e^{-i}$ the number of agents other than $i$ choosing edge $e$, we have that 
\begin{align*}
\Phi^{(\pi_i,\pi_{-i})}(s) = &\sum_{e \text{ reachable from }s}\pi_{i}(e|s)\mathbb{E}_{\pi_{-i}}\left[ \sum_{k=1}^{l_e^{-i}+1} c_e(k)\right]+(1-\pi_{i}(e|s))\mathbb{E}_{\pi_{-i}}\left[ \sum_{k=1}^{l_e^{-i}} c_e(k)\right] \\
= &\sum_{e \text{ reachable from }s}\pi_{i}(e|s)\mathbb{E}_{\pi_{-i}}\left[  c_e(l_e+1)\right] +\sum_{e \text{ reachable from }s}\mathbb{E}_{\pi_{-i}}\left[ \sum_{k=1}^{l_e} c_e(k)\right]. 
\end{align*}

So if player $i$ deviates from $\pi_i$ then the change in the potential will equal the change in their individual utility since the common term in the second sum cancels as it does not depend on $\pi_i$.
We refer to this illustrative class of MPGs zas stochastic congestion games. 
\end{example}

\subsection{Natural Policy Gradient}
\paragraph{Softmax parametrization.} We assume that each agent $i$ is represented by its logits — a vector $\theta_{i}$ of size $\mathcal S \times \mathcal A_i$ such that the policy is given by 
\begin{equation}
    \pi_{\theta_i}(a|s) = \frac{\exp (\theta_{i,s,a})}{\sum_{a' \in \mathcal A_i}\exp (\theta_{i,s,a'})} \textrm{ for all }a \in \mathcal A_{i}, s\in \mathcal S. 
\end{equation}

\paragraph{Independent Natural Policy Gradient (INPG).} Given the softmax parametrization, we can define Independent Natural Policy Gradient:
\begin{equation}\label{eq:npggeneral}
    \theta^{(t+1)}_{i} = \theta^{(t)}_i+\eta F_{i,\mu}^{+}(\theta^{(t)}) \nabla_{\theta_i} V_i^{\pi_{\theta^{(t)}}} \textrm{ for each agent }i,
\end{equation}
where $F_{i,\mu}^{+}(\theta)$ is the pseudo-inverse of the Fischer matrix 
\begin{align*}
F_{i,\mu}(\theta) = \mathbb{E}_{s\sim d^{\pi_{\theta}}_{\mu}} \mathbb{E}_{a_i \sim \pi_{\theta_i}(.|s)} \left[\nabla_{\theta_i} \log \pi_{\theta_i}(a|s) \nabla_{\theta_i} \log \pi_{\theta_i}(a|s) ^{\top}\right]
\end{align*}
and $\eta$ is the step-size. 

NPG equations can be simplified using the advantage function:
\begin{lemma}[Simplified equations of INPG \cite{Aga20}]\label{lem:npgequations}
For each agent $i$, state $s$, and action $a$, we have that
\begin{equation}\label{eq:npg}\tag{INPG}
    \theta^{(t+1)}_{i,s,a} = \theta_{i,s,a}^{(t)}+\frac{\eta}{1-\gamma}A_i^{(t)}(s,a),
\end{equation}
where $A_{i}^{\theta}(s,a) = \mathbb{E}_{\a_{-i}\sim \pi_{\theta_{-i}}}[A_i^{\pi_{\theta}}(s,a,\a_{-i})]$ is the expectation of the advantage function of $i$ over the actions of all agents but $i$ who use joint policy $\pi_{\theta_{-i}}$, and $A_i^{(t)}$ is shorthand for $A_i^{\theta^{(t)}}$. 
\end{lemma}

\begin{remark}\label{rem:fixedpoints} If we make a change of variables from $\theta$ to probabilities $\pi$, the equations of \ref{eq:npg} become the standard Multiplicative Weights Update algorithm (MWU), that is
\begin{equation}\label{eq:mwua}\tag{MWU}
    \pi_i^{(t+1)}(a|s) =  \pi_i^{(t)}(a|s)\frac{\exp (\frac{\eta}{1-\gamma} A^{\pi}_{i}(s,a))}{Z^{(t)}_{i,s}},
\end{equation}
where $Z_{i,s}^{(t)}$ is the renormalization term. Observe that the fixed points $\pi$ of (\ref{eq:mwua}) satisfy either $\pi_{i} (a|s)=0$ or $A^{\pi}_i(s,a)=0$ for all $a\in \mathcal A_i, s\in \mathcal S$ and agents $i$. 
\end{remark}
\section{Convergence results}
In this section we prove that INPG converges. The main theorem is stated formally as follows:
\begin{theorem}\label{thm:mainconvergence} Given a Markov Potential Game, assume that all the agents use Independent Natural Policy Gradient (\ref{eq:npg}) with stepsize $\eta < \frac{(1-\gamma)^3}{27n^2 A^2_{\max}M}$. It holds that the dynamics (\ref{eq:npg}) converge pointwise (last iterate convergence) to equilibrium policies.
\end{theorem}

 The first step towards our result is to show that INPG is equivalent to running Natural Policy Gradient (NPG) on the potential function. 
\begin{lemma}\label{lem:equivalence}
We consider Natural Policy Gradient dynamics on function $\Phi^{\pi_{\theta}}$, that gives
\begin{equation}\label{eq:npgpotential}
\theta^{(t+1)} = \theta^{(t)} +\eta F_{\mu}^{+}(\theta^{(t)}) \nabla_{\theta} \Phi^{\pi_{\theta^{(t)}}}
\end{equation}
where $F_{\mu}^{+}(\theta)$ is the pseudo-inverse of the Fischer matrix 
\begin{align*}
F_{\mu}(\theta) = \mathbb{E}_{s\sim d^{\pi_{\theta}}_{\mu}} \mathbb{E}_{\a \sim \pi_{\theta}(.|s)}  
&\left[\nabla_{\theta} \log \pi_{\theta}(\a|s) \nabla_{\theta} \log \pi_{\theta}(\a|s) ^{\top}\right].
\end{align*}

It is true that (\ref{eq:npg}) is the same dynamics as (\ref{eq:npgpotential}).
\end{lemma}
\begin{proof}
The first fact we are going to use is that
\begin{equation}\label{eq:potentialtovalue}
\nabla_{\theta_i} \Phi^{\pi_{\theta}} = \nabla_{\theta_i} V_{i}^{\pi_{\theta}}. 
\end{equation}
This is true since 
$\nabla_{\pi_{\theta_i}} \Phi^{\pi_{\theta}}= \nabla_{\pi_{\theta_i}} V_i^{\pi_{\theta}}$ (Proposition B.1  in \cite{leonardos21}) and by applying the chain rule. Moreover, since $\pi_{\theta}$ is a product distribution, we have 
\begin{equation}\label{eq:log}
\log \pi_{\theta}(\a|s) = \sum_i \log \pi_{\theta_i}(a_i|s),
\end{equation}
where $\a = (a_1,...,a_n).$ Assuming $i\neq j$, we have
\begin{align}\label{eq:zero}
    \mathbb{E}_{\a\sim\pi_{\theta}} &[\nabla_{\theta} \pi_{\theta_i}(a_i|s) \nabla_{\theta}\pi_{\theta_j}(a_j|s)^{\top}] = \nonumber \\    &\mathbb{E}_{\a\sim\pi_{\theta}}[\nabla_{\theta_i} \pi_{\theta_i}(a_i|s)]\mathbb{E}_{\a\sim\pi_{\theta}} \left[\nabla_{\theta}\pi_{\theta_j}(a_j|s)\right]^{\top}=0, 
\end{align} 
where the first equality comes from independence and the last equality from the fact that the expectation of the derivative of the log-likelihood is zero. We conclude that
\begin{align*}
F_{\mu}(\theta) &= \mathbb{E}_{s\sim d^{\pi_{\theta}}_{\mu}} \mathbb{E}_{\a \sim \pi_{\theta}(.|s)} \left[\nabla_{\theta} \log \pi_{\theta}(a|s) \nabla_{\theta} \log \pi_{\theta}(a|s) ^{\top}\right]\\&\stackrel{(\ref{eq:log}),(\ref{eq:zero})}{=} \sum_{i \in N} \mathbb{E}_{s\sim d^{\pi_{\theta}}_{\mu}} \mathbb{E}_{\a \sim \pi_{\theta}(.|s)} \left[\nabla_{\theta} \log \pi_{\theta_i}(a_i|s) \nabla_{\theta} \log \pi_{\theta_i}(a_i|s) ^{\top}\right]\\&\stackrel{\textrm{prod.distr.}}{=}
\sum_{i \in N} \mathbb{E}_{s\sim d^{\pi_{\theta}}_{\mu}} \mathbb{E}_{\a_i \sim \pi_{\theta_i}(.|s)} \left[\nabla_{\theta} \log \pi_{\theta_i}(a_i|s) \nabla_{\theta} \log \pi_{\theta_i}(a_i|s) ^{\top}\right]
,
\end{align*}
i.e., the Fisher matrix $F_{\mu}(\theta)$ is block-diagonal with each block being the Fisher matrix of each individual player.
Finally, from Lemma 5.1 in \cite{Aga20} it holds that 
\begin{equation}\label{eq:Aga20}
F^{+}_{i,\mu}(\theta)\nabla_{\theta_i} V_{i}^{\pi_{\theta}}  = \frac{1}{1-\gamma} \mathbb{E}_{\a_{-i}\sim \pi_{\theta_{-i}}}A_{i}^{\pi_{\theta}}(s,a,\a_{-i}).
\end{equation}
The proof is now concluded by (\ref{eq:potentialtovalue}), (\ref{eq:Aga20}), and the fact that the pseudo-inverse of a block-diagonal matrix is block-diagonal with blocks that are the pseudo-inverse of each block of the original matrix.
\end{proof}
With Lemma \ref{lem:equivalence} established, we next show that the potential function is non-decreasing along the dynamics (\ref{eq:npg}).
This requires the following proposition, the proof of which can be found in the appendix:
\begin{proposition}[Smoothness — Mahalanobis]\label{cl:mahalanobis}
Let $\theta , \tilde{\theta}$ be such that $\norm{\theta-\tilde{\theta}}_{\infty} \leq \frac{\eta}{1-\gamma}.$ There exists a constant $L = \frac{27n^2 A^2_{\max}M}{(1-\gamma)^3}$ so that
\begin{align*}
-\Phi^{\pi_{\tilde{\theta}}}(\mu) &\leq -\Phi^{\pi_{\theta}}(\mu) + \langle -\nabla _{\theta}\Phi^{\pi_{\theta}}(\mu), \tilde{\theta}-\theta \rangle + frac{L}{2}\norm{\tilde{\theta}-\theta}^2_{D^{\theta}},
\end{align*}
where $D^{\theta}$ is a diagonal matrix of size $\left(\sum_{i} |S| \cdot |A_{i}|\right)\times \left(\sum_{i} |S| \cdot |A_{i}|\right)$ with diagonal entry $(i,s,a), (i,s,a)$ to be $d^{\pi_{\theta}}_{\mu}(s) \pi_{\theta_i}(a|s)$ and $A_{\max} = \max_{i\in N} |A_{i}|.$
\end{proposition}

\begin{lemma}[$\Phi^{(t)}$ is non-decreasing]\label{lem:increasingphi} Assume that the agents use the dynamics (\ref{eq:npg}). It holds that
\begin{equation}\label{eq:phiincreasing}
    \Phi^{(t+1)}(\mu) \geq \Phi^{(t)}(\mu) \textrm{ for all } t.
\end{equation}
\end{lemma}
\begin{proof}
Using Lemma \ref{lem:equivalence}, it suffices to focus on the dynamics (\ref{eq:npgpotential}), i.e., we will analyze Natural Policy Gradient on the potential function $\Phi^{\pi_{\theta}}(\mu)$. For agent $i$ it holds that
\begin{align*}
 \frac{\partial \Phi^{\pi_{\theta}}}{\partial \theta_{i,s,a}} &=  \frac{\partial V_i^{\pi_{\theta}}}{\partial \theta_{i,s,a}} = d^{\pi_{\theta}}_{\mu}(s)\pi_{\theta_i}(a|s) \mathbb{E}_{\a_{-i}\sim \pi_{\theta_{-i}}}A_i^{\pi_{\theta}}(s,a,\a_{-i}),
\end{align*}
hence by defining $D^{(t)}$ to be a diagonal matrix of size $\left(\sum_{i} |\mathcal S| \times |\mathcal A_{i}|\right)^2$ with diagonal entry $(i,s,a_{i}), (i,s,a_{i})$ to be $d^{\pi_{\theta^{(t)}}}_{\mu}(s) \pi_{\theta^{(t)}_i}(a_i|s)$ we have
that NPG on $\Phi$ is the same as
\begin{equation}\label{eq:diagonal}
\theta^{(t+1)} = \theta^{(t)} + \eta \left(D^{(t)}\right)^{ -1} \nabla_{\theta}\Phi^{(t)}(\mu).
\end{equation}
Since all rewards are in $[0,1]$ it follows that $\norm{\theta^{(t+1)} - \theta^{(t)}}_{\infty} \leq \frac{\eta}{1-\gamma}.$ 
Given \Cref{cl:mahalanobis}, 
it follows that
\begin{align*}
    -\Phi^{(t+1)}(\mu) &\leq -\Phi^{(t)}(\mu) + \langle -\nabla _{\theta}\Phi^{(t)}(\mu), \theta^{(t+1)}-\theta^{(t)} \rangle + \frac{L}{2}\norm{\theta^{(t+1)}-\theta^{(t)}}^2_{D^{(t)}}\\&=
    -\Phi^{(t)}(\mu) - \eta \norm{\nabla _{\theta}\Phi^{(t)}(\mu)}^2_{\left(D^{(t)}\right)^{-1}}+ \frac{\eta^2 L}{2}\norm{\nabla _{\theta}\Phi^{(t)}(\mu)}^2_{\left(D^{(t)}\right)^{-1}}. 
\end{align*}
Therefore, by choosing $\eta < \frac{1}{L}$ we have that
\begin{equation}
    \Phi^{(t+1)} \geq \Phi^{(t)} + \frac{1}{2L}\norm{\nabla _{\theta}\Phi^{(t)}(\mu)}^2_{\left(D^{(t)}\right)^{-1}} > \Phi^{(t)} ,
\end{equation}
and the proof is complete. 
\end{proof}

From Lemma \ref{lem:increasingphi} and the fact that $\Phi$ is defined on a compact domain and is continuous (hence attains a maximum) we conclude that if $\eta < \frac{(1-\gamma)^3}{27n^2 A^2_{\max}M}$ then $\Phi^{(t)}$ converges to a limit $\Phi^{(\infty)}$. We shall show that the corresponding policies are indeed equilibrium policies. First, we prove that the limit $\Phi^{(\infty)}$ is a fixed point of the Independent Natural Policy Gradient dynamics.

\begin{lemma}[Convergence to fixed points]\label{lem:cormain} INPG dynamics (\ref{eq:npg}) converges to fixed points.
\end{lemma}

\begin{proof}
Let $\Omega \subset \Delta(\mathcal A)^{|\mathcal S|}$ be the set of limit points of  $\pi_{\theta^{(t)}}$.  $\Phi^{(t)}$ is increasing with respect to time $t$ by Lemma \ref{lem:increasingphi} and so, because $\Phi$ is bounded on $\Delta(\mathcal A)^{|\mathcal S|}$, $\Phi^{(t)}$ converges as $t\to \infty$  to $\Phi^{(\infty)}= \inf_t\{\Phi^{(t)}\}$. By continuity of $\Phi$ we get that \mbox{$\Phi^{q}= \lim_{t\to\infty} \Phi^{q_{\theta^{(t)}}} = \Phi^{(\infty)}$} for all $q \in \Omega$. So $\Phi$ is constant on $\Omega$. Also \mbox{$q^{(t)} = \lim_{n \to \infty} q^{(t_n + t)}$} as $n \to \infty $ for some sequence of times $\{t_i\}$ and so $q^{(t)}$ lies in $\Omega$, i.e. $\Omega$ is invariant. Thus, if $q \equiv q^{(0)} \in \Omega$ the trajectory $q^{(t)}$ lies in $\Omega$ and so $\Phi^{q^{(t)}} = \Phi^{\infty}$ on the trajectory. But $\Phi$ is strictly increasing except on fixed points and so $\Omega$ consists entirely of fixed points.
\end{proof}

We conclude the section by showing that if $\pi^{(0)}$ is in the interior of $\Delta(\mathcal A)^{|\mathcal S|}$ (all $\theta$ are bounded at initialization) then we have convergence to equilibrium policies.
\begin{lemma}[Convergence to equilibrium policies] \label{lem:convergencenash} Assume that the fixed points of (\ref{eq:npg}) are isolated. Let $\pi^{(0)}$ be a point in the interior of $\Delta(\mathcal A)^{|\mathcal S|}$. It follows that $\lim_{t \to \infty} \pi^{(t)} = \pi^{(\infty)}$ is an equilibirum policy.
\end{lemma}
\begin{proof} We showed in Lemma \ref{lem:cormain} that INPG dynamics (\ref{eq:npg}) converges, hence $\lim_{t \to \infty} \pi^{(t)}$ exists (under the assumption that the fixed points are isolated) and is equal to a fixed point of the dynamics $\pi^{(\infty)}$. Also it is clear from the dynamics that $\Delta(\mathcal A)^{|\mathcal S|}$ is invariant, i.e., for all $t\geq 0$, we have $\sum_{a\in \mathcal A_i}\pi_i^{(t)}(a|s)=1$ for all $s\in \mathcal S$, $i\in N$ and $\pi^{(t)}_i(a|s)>0$ for all $a \in \mathcal A_i,s \in \mathcal S$, and $i \in N$ (since $\pi^{(0)}$ is in the interior of $\Delta(\mathcal A)^{|\mathcal S|}$).

Assume that $\pi^{(\infty)}$ is not an equilibrium policy, then there exists an agent $i$, a state $s$ and an action $a \in \mathcal A_i$ so that $A^{\pi^{(\infty)}}_{i}(s,a)>0$ and $\pi^{(\infty)}_i(a|s) =0.$ Fix a $\zeta>0$ and let $U_{\zeta}=\{\pi:\frac{\eta A^{\pi}_{i}(s,a)}{1-\gamma} >\zeta + \log Z^{\pi}_{i,s}\}$. By continuity we have that $U_{\zeta}$ is open. It is also true that $\pi^{(\infty)} \in U_{\zeta}$ for $\zeta$ small enough since $Z^{\pi^{(\infty)}}_{i,s} =1$.

Since $\pi^{(t)}$ converges to $\pi^{(\infty)}$ as $t\to \infty$, there exists a time $t_{0}$ so that for all $t' \geq t_0$ we have that $\pi^{(t')} \in U_{\zeta}$. However, from INPG dynamics (\ref{eq:npg}) we get that if $\pi^{(t')}\in  U_{\zeta}$ then $e^{\frac{\eta A^{(t')}_{i}(s,a) }{1-\gamma}}\geq e^{\zeta} Z^{(t')}_{i,s} \geq Z^{(t')}_{i,s}.$  Hence
$$\pi_{i}^{(t'+1)}(a|s) = \pi_{i}^{(t')}(a|s) \frac{\exp( \frac{\eta A^{(t)}_i(s,a)}{1-\gamma} )}{Z_{i,s}^{(t')}} \geq \pi_i^{(t')}(a|s)>0,$$

i.e., $\pi_{i}^{(t')}(a|s)$ is positive and increasing with \mbox{$t' \geq t_0$}. We reached a contradiction since \mbox{$\pi_{i}^{(t)}(a|s) \to \pi^{(\infty)}_{i}(a|s)=0$}, thus $\pi^{(\infty)}$ is an equilibrium policy.
\end{proof}

The proof of Theorem \ref{thm:mainconvergence} follows by Lemmas \ref{lem:cormain} and \ref{lem:convergencenash}.

\section{Experiments}

In this section we compare the empirical performance of Independent Natural Policy Gradient (INPG) to Independent Policy Gradient (IPG) in both the Stochastic Congestion Game (SCG) environment described in the Preliminaries and the environment introduced by \cite{leonardos21}, which we will call the \emph{distancing game}. We compare INPG to IPG in the distancing game setting to show empirical improvement over previous results, but also include this new SCG setting for several reasons. The distancing game considered in \cite{leonardos21} had only 2 states and the required number of iterations for the experiments to empirically converge (for policies to not change by more than $1e-16$ between iterates) was far less than theoretically guaranteed. The SCG setting we use, which has more states yet less agents, required many more iterations of IPG, demonstrating a significant gap in the empirically observed convergence rates of IPG and INPG. Experiments were run on a 2021 MacBook Air.

\subsection{Distancing Game}

For completeness, we briefly reproduce the description of the distancing game as given in \cite{leonardos21}.

\paragraph{Experimental setup:} We consider a MDP in which every state is a congestion game \cite{rosenthal73}. There are 8 agents, 4 facilities that the agents can select from and 2 states: a \emph{safe} state and a \emph{spread} state. In both states, the agents prefer to share their facility with as many other agents as possible. Specifically, the reward of each agent for choosing the action corresponding to facility $k$ is equal to a positive weight $w_k^{\text{safe}}$ times the number of agents at $k=A,B,C,D$. The weights satisfy $w_D^{\text{safe}}>w_C^{\text{safe}}>w_B^{\text{safe}}>w_A^{\text{safe}}$, i.e., facility $D$ is the most preferred facility by all agents. However, if more than 4 agents choose the same facility, then the MDP transitions to the spread state. The weight function is the same for all agents at the spread state as well, but the reward is reduced by a large constant amount $c>0$. To return to the safe state, the agents need to maximally spread out amongst the facilities, that is, no more than 2 agents may be in the same facility. 

\paragraph{Independent Policy Gradient Parameters:} We perform episodic updates with $T=20$ steps. At each iteration, we estimate the Q-functions, the value function, the discounted visitation distributions and, hence, the policy gradients using the average of mini-batches of size $20$. We use discount factor $\gamma=0.99$ and learning rate $\eta=0.0001$. Empirical convergence was determined to have occurred if the L1 distance between the consecutive policies of each agent was less than $10e-16$. 

\paragraph{Independent Natural Policy Gradient Parameters:} We use the same parameters for the INPG implementation as the IPG implementation. Specifically, we perform episodic updates with $T=20$ steps. We estimate the Q-functions and and value function at each iteration, which allows us to obtain an estimate for the advantage function that is used in the INPG update. This is done using mini-batches of size 20. Again the discount factor is $\gamma=0.99$ and the learning rate is $\eta=0.0001$. Empirical convergence was determined to have occurred if the L1 distance between the consecutive policies of each agent was less than $10e-16$. 

\paragraph{Results:}  \Cref{fig:ind_dg} and \Cref{fig:avg_dg} depict the L1-accuracy in the policy space at each iteration, i.e., $\text{L1-accuracy}=\frac1N\sum_{i\in \N}|\pi_i-\pi_i^{\text{final}}|=\frac1N\sum_{i\in \N}\sum_{s}\sum_{a}|\pi_i(a\mid s)-\pi_i^{\text{final}}(a\mid s)|.$ This is the average distance between the current policy and the final policy of all agents.

\begin{figure}[h]
\vspace{.05in}
\centerline{\includegraphics[width=.75\textwidth]{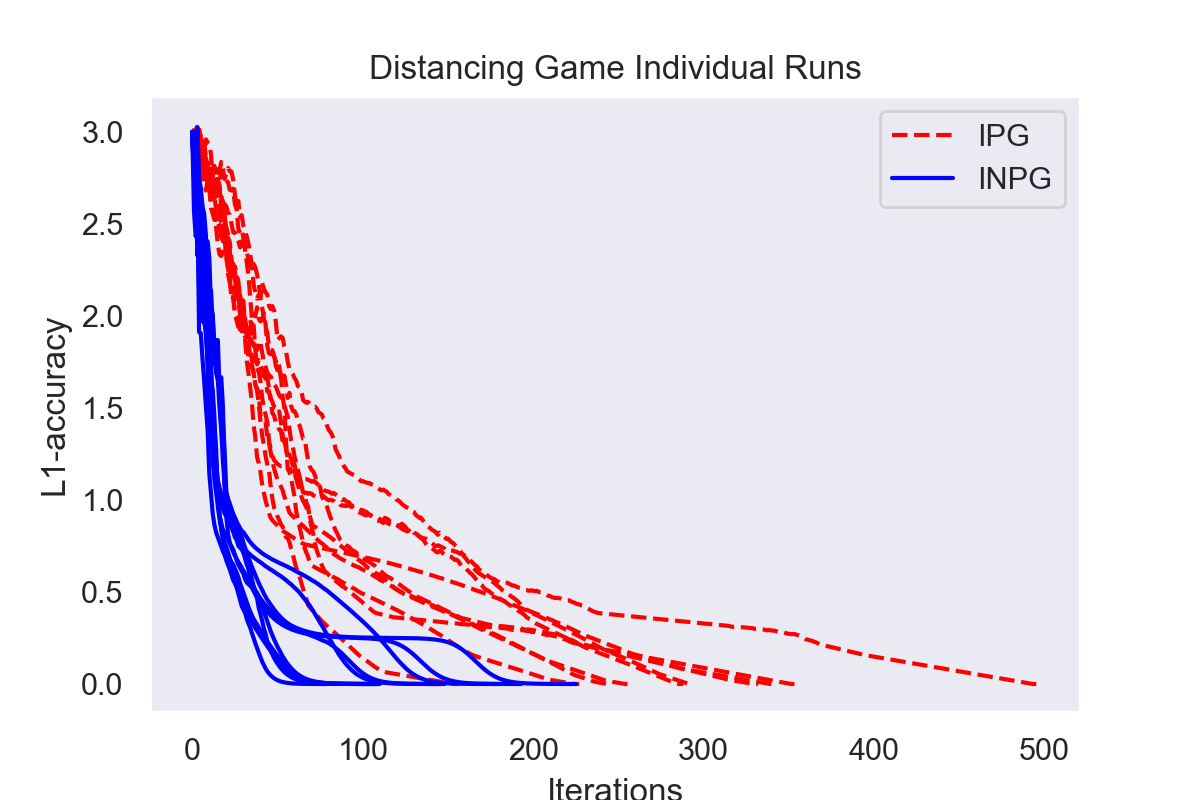}}
\vspace{.2in}
\caption{Trajectories of the L1-accuracy for 10 runs each of Independent Policy Gradient and Independent Natural Policy Gradient in the distancing game environment with 8 agents and learning rate of $\eta=.0001$.}
\label{fig:ind_dg}
\end{figure}

\begin{figure}[h]
\centerline{\includegraphics[width=0.75\textwidth]{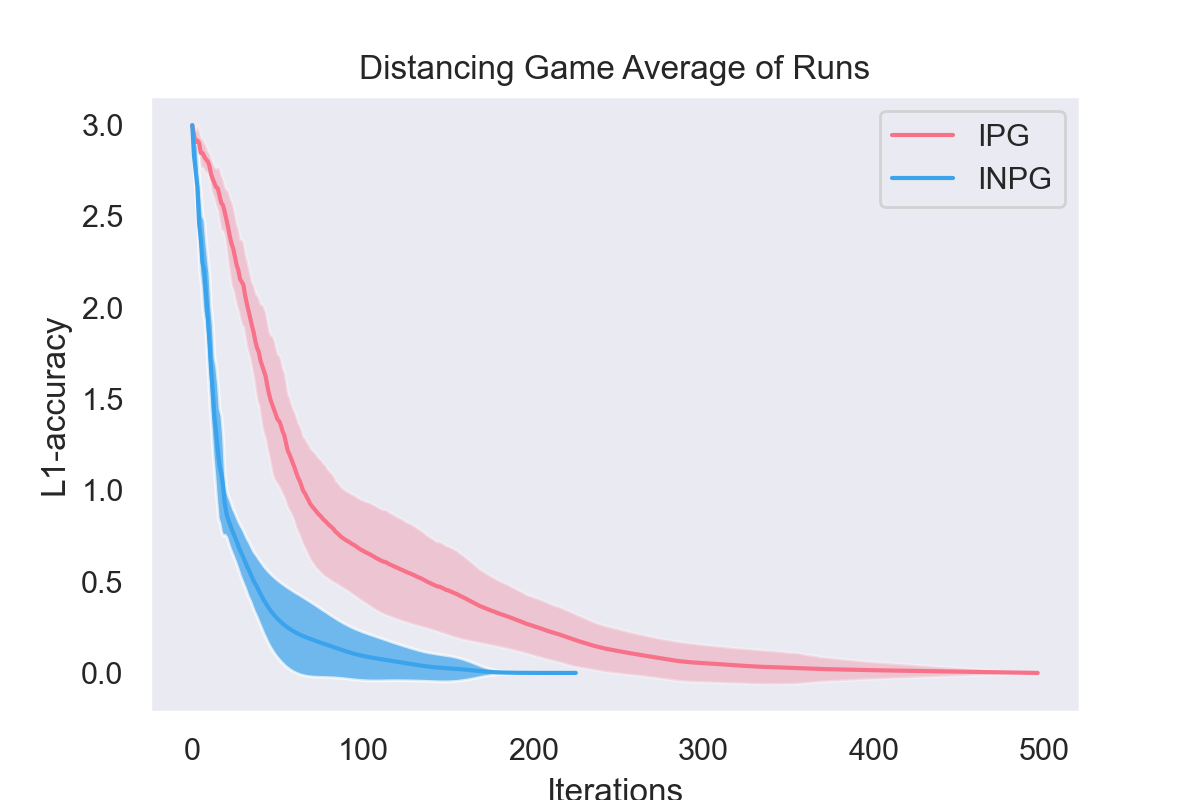}}
\vspace{.2in}
\caption{Mean L1-accuracy with shaded region of one standard deviation for the trajectories of \Cref{fig:ind_dg}.}
\label{fig:avg_dg}
\end{figure}

We can see from \Cref{fig:ind_dg} and \Cref{fig:avg_dg} that on average and in almost every run, Independent Natural Policy Gradient converges faster than Independent Policy Gradient. This matches our intuitive expectation that NPG methods should require fewer iterations than PG methods to achieve convergence \cite{Aga20}.

\subsection{Stochastic Congestion Game}

We now compare Independent Policy Gradient and Independent Natural Policy Gradient for the environment described in \Cref{ex:scg}.

\paragraph{Experimental setup:} We consider a MDP where each state corresponds to a configuration of $n$ agents on a graph such as in \Cref{fig:scg}, except that the internal layers are fully connected.The graph is directed and the edges only move from left to right. It is important to note that the states of the MDP are not the vertices of this graph but the entire configuration, which is captured by the location of each agent. Thus there are $|V|^n$ possible states in the MDP where $|V|$ is the number of vertices in the DAG. 

The actions of each agent at a given state of the MDP are to choose the edge from their current, corresponding vertex in the DAG that they want to move along. They then receive a reward inversely proportional to the number of agents who chose this edge and the MDP transitions to the next configuration state. If the agents reach $t$ we can either send them along a constant reward edge back to $s$ or consider the episode to be terminated. 

The specific DAG on which the experiments were run was with 2 internal fully-connected layers of 2 vertices each, giving the graph 6 total vertices consisting of $s$, $t$, and four internal vertices. Thus at each state each agent has two possible actions to choose from, corresponding to the two edges coming out of their vertex in the DAG. The specific value of $n$ chosen for comparing independent policy gradient and independent natural policy gradient was 4. This value was chosen due to the fact that for $n \geq 5$ agents, IPG required more than 3000 iterations, each of which require a nontrivial amount of time due to needing to make tabular updates for each state-action pair. 

\paragraph{Parameters:} All parameters and methods were the same as for IPG and INPG in the distancing game setting of the previous subsection. 

\paragraph{Results:} \Cref{fig:ind_scg} and \Cref{fig:avg_scg} depict the L1-accuracy in the policy space at each iteration, i.e., $\text{L1-accuracy}=\frac1N\sum_{i\in \N}|\pi_i-\pi_i^{\text{final}}|=\frac1N\sum_{i\in \N}\sum_{s}\sum_{a}|\pi_i(a\mid s)-\pi_i^{\text{final}}(a\mid s)|.$ This is the average distance between the current policy and the final policy of all agents.

\begin{figure}[h]
\vspace{.05in}
\centerline{\includegraphics[width=0.75\textwidth]{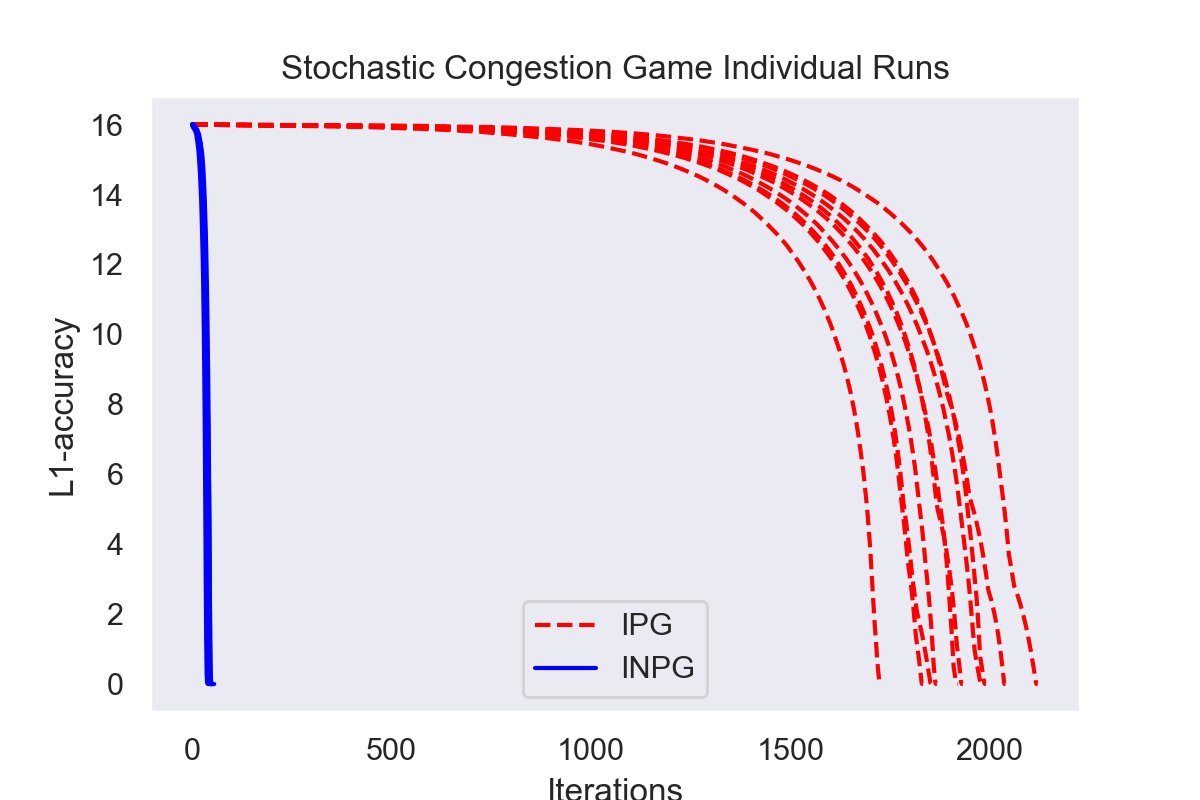}}
\vspace{.2in}
\caption{Trajectories of the L1-accuracy for 10 runs each of Independent Policy Gradient and Independent Natural Policy Gradient in the stochastic congestion game environment with 4 agents and learning rate of $\eta=.0001$.}
\label{fig:ind_scg}
\end{figure}

\begin{figure}[h]
\centerline{\includegraphics[width=0.75\textwidth]{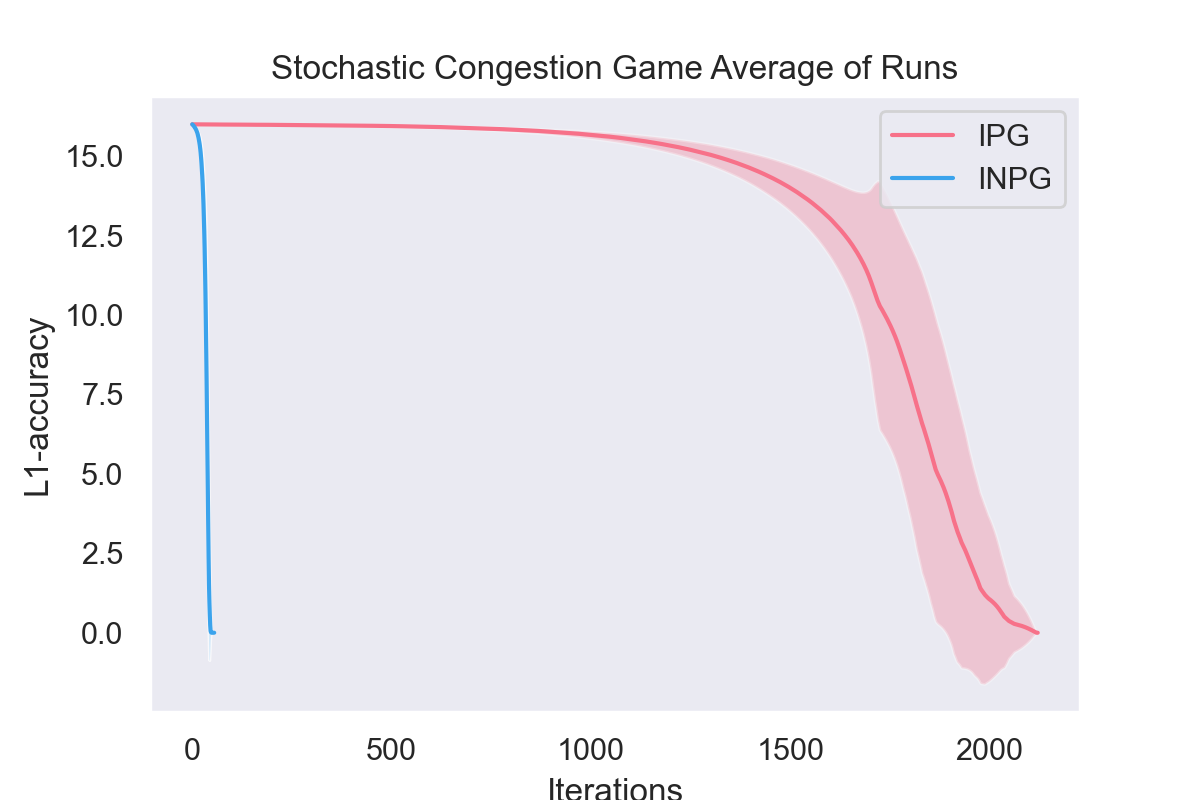}}
\vspace{.2in}
\caption{Mean L1-accuracy with shaded region of one standard deviation for the trajectories of \Cref{fig:ind_scg}.}
\label{fig:avg_scg}
\end{figure}

These results show a dramatic difference in performance of IPG and INPG in the stochastic congestion game environment. The agents running IPG took about 2000 iterations to converge, while the agents running INPG took only about 50. This experiment also shows that a key distinguishing aspect between single-agent PG and NPG also can appear in the multi-agent case; IPG suffers from a long flat region of little improvement when far from a Nash policy, while INPG quickly moves the agents towards a Nash policy. This is in line with observed differences between single-agent PG and NPG moving towards optimal policies \cite{Kak01}.

Furthermore, as mentioned above, IPG was not able to converge within 3000 iterations for $n \geq 5$ agents in the stochastic congestion game setting. In comparison, as see in \Cref{fig:ing_scg8}, INPG was able to converge in less than 100 iterations even with 8 agents. 

\begin{figure}[h]
\centerline{\includegraphics[width=0.75\textwidth]{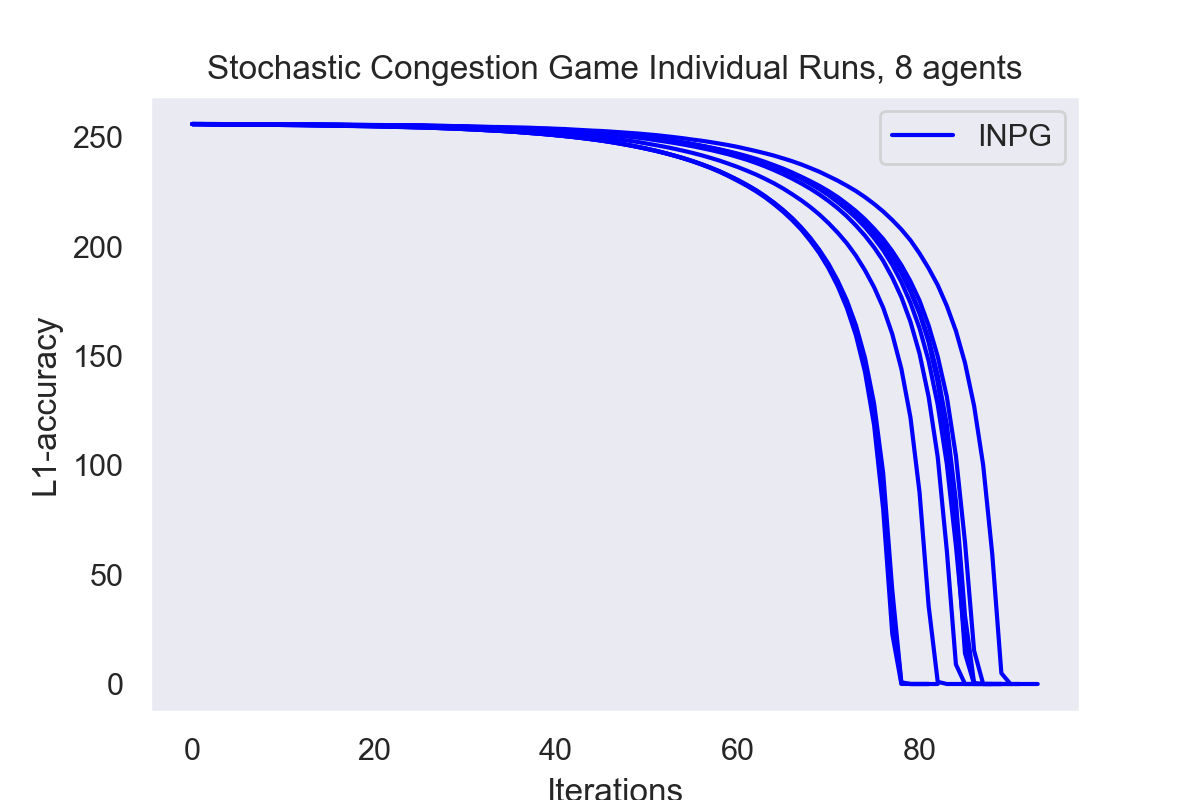}}
\vspace{.2in}
\caption{Trajectories of the L1-accuracy for 10 runs each of independent natural policy gradient in the stochastic congestion game environment with 8 agents and learning rate of $\eta=.0001$.}
\label{fig:ing_scg8}
\end{figure}

These experimental results highlight the key reason that guarantees regarding the convergence of independent natural policy gradient are necessary to obtain; INPG generally outperforms IPG in all of our experiments, in some cases by a significantly large margin. Again, this is in line with expectations based on the performance of single-agent policy gradient versus natural policy gradient \cite{Aga20}.

\section{Conclusion}
In this paper we showed that independent Natural Policy Gradient will always converge to a Nash equilibrium in Markov Potential Games under mild conditions. We also show experimental results where NPG converges much faster than independent policy gradients. 

In future work, we will prove rates of convergence for NPG in MPGs. We will also look into having each learner estimate advantages instead of receiving oracle advantages and will prove results with function approximation. It is also unknown whether trust region methods like TRPO will converge in MPGs which generalize cooperative games, for which TRPO convergence has been established \cite{kuba2021trust}. Going beyond this work, a large open problem remain proving (or disproving) the convergence of independent RL agents to correlated equilibrium in general-sum Markov games. 

\newpage

\section{Acknowledgements}

Roy Fox's research is partly funding by the Hasso Plattner Foundation. Ioannis Panageas and Will Overman are supported by a startup grant.

\bibliographystyle{plain}
\bibliography{bib_markov_potential}

\begin{thebibliography}{10}

\bibitem{Aga20}
A.~Agarwal, S.~M. Kakade, J.~D. Lee, and G.~Mahajan.
\newblock {Optimality and Approximation with Policy Gradient Methods in Markov
  Decision Processes}.
\newblock In J.~Abernethy and S.~Agarwal, editors, {\em Proceedings of 33rd
  Conference on Learning Theory}, volume 125 of {\em PMLR}, pages 64--66, 2020.

\bibitem{akkaya2019solving}
Ilge Akkaya, Marcin Andrychowicz, Maciek Chociej, Mateusz Litwin, Bob McGrew,
  Arthur Petron, Alex Paino, Matthias Plappert, Glenn Powell, Raphael Ribas,
  et~al.
\newblock Solving rubik's cube with a robot hand.
\newblock {\em arXiv preprint arXiv:1910.07113}, 2019.

\bibitem{arulkumaran2017deep}
Kai Arulkumaran, Marc~Peter Deisenroth, Miles Brundage, and Anil~Anthony
  Bharath.
\newblock Deep reinforcement learning: A brief survey.
\newblock {\em IEEE Signal Processing Magazine}, 34(6):26--38, 2017.

\bibitem{brown2019deep}
Noam Brown, Adam Lerer, Sam Gross, and Tuomas Sandholm.
\newblock Deep counterfactual regret minimization.
\newblock In {\em International conference on machine learning}, pages
  793--802. PMLR, 2019.

\bibitem{brown2018superhuman}
Noam Brown and Tuomas Sandholm.
\newblock Superhuman ai for heads-up no-limit poker: Libratus beats top
  professionals.
\newblock {\em Science}, 359(6374):418--424, 2018.

\bibitem{daskalakis2021independent}
Constantinos Daskalakis, Dylan~J Foster, and Noah Golowich.
\newblock Independent policy gradient methods for competitive reinforcement
  learning.
\newblock {\em Neural Information Processing Systems (NeurIPS) 2020}, 2020.

\bibitem{foerster2018counterfactual}
Jakob Foerster, Gregory Farquhar, Triantafyllos Afouras, Nantas Nardelli, and
  Shimon Whiteson.
\newblock Counterfactual multi-agent policy gradients.
\newblock In {\em Proceedings of the AAAI Conference on Artificial
  Intelligence}, volume~32, 2018.

\bibitem{foerster2017learning}
Jakob~N Foerster, Richard~Y Chen, Maruan Al-Shedivat, Shimon Whiteson, Pieter
  Abbeel, and Igor Mordatch.
\newblock Learning with opponent-learning awareness.
\newblock {\em International Conference on Autonomous Agents and Multiagent
  Systems (AAMAS)}, 2017.

\bibitem{heinrich2016deep}
Johannes Heinrich and David Silver.
\newblock Deep reinforcement learning from self-play in imperfect-information
  games.
\newblock {\em arXiv preprint arXiv:1603.01121}, 2016.

\bibitem{hennes2019neural}
Daniel Hennes, Dustin Morrill, Shayegan Omidshafiei, Remi Munos, Julien
  Perolat, Marc Lanctot, Audrunas Gruslys, Jean-Baptiste Lespiau, Paavo Parmas,
  Edgar Duenez-Guzman, et~al.
\newblock Neural replicator dynamics.
\newblock {\em AAMAS 2020}, 2019.

\bibitem{Kak01}
Sham Kakade.
\newblock A natural policy gradient.
\newblock In {\em Proceedings of the 14th International Conference on Neural
  Information Processing Systems: Natural and Synthetic}, NIPS'01, page
  1531–1538, Cambridge, MA, USA, 2001. MIT Press.

\bibitem{kuba2021trust}
Jakub~Grudzien Kuba, Ruiqing Chen, Munning Wen, Ying Wen, Fanglei Sun, Jun
  Wang, and Yaodong Yang.
\newblock Trust region policy optimisation in multi-agent reinforcement
  learning.
\newblock {\em arXiv preprint arXiv:2109.11251}, 2021.

\bibitem{lanctot2017unified}
Marc Lanctot, Vinicius Zambaldi, Audrunas Gruslys, Angeliki Lazaridou, Karl
  Tuyls, Julien P{\'e}rolat, David Silver, and Thore Graepel.
\newblock A unified game-theoretic approach to multiagent reinforcement
  learning.
\newblock {\em Neural Information Processing Systems (NIPS)}, 2017.

\bibitem{leonardos21}
Stefanos Leonardos, Will Overman, Ioannis Panageas, and Georgios Piliouras.
\newblock Global convergence of multi-agent policy gradient in markov potential
  games.
\newblock {\em CoRR}, abs/2106.01969, 2021.

\bibitem{lowe2017multi}
Ryan Lowe, Yi~Wu, Aviv Tamar, Jean Harb, Pieter Abbeel, and Igor Mordatch.
\newblock Multi-agent actor-critic for mixed cooperative-competitive
  environments.
\newblock {\em Neural Information Processing Systems (NIPS)}, 2017.

\bibitem{mcaleer2021xdo}
Stephen McAleer, John Lanier, Pierre Baldi, and Roy Fox.
\newblock Xdo: A double oracle algorithm for extensive-form games.
\newblock {\em arXiv preprint arXiv:2103.06426}, 2021.

\bibitem{mcaleer2020pipeline}
Stephen McAleer, John Lanier, Roy Fox, and Pierre Baldi.
\newblock Pipeline psro: A scalable approach for finding approximate nash
  equilibria in large games.
\newblock {\em Neural Information Processing Systems (NeurIPS)}, 2020.

\bibitem{mgu20}
D.~{Mguni}.
\newblock {Stochastic Potential Games}.
\newblock {\em arXiv e-prints}, page arXiv:2005.13527, May 2020.

\bibitem{mguni2021learning}
David Mguni, Yutong Wu, Yali Du, Yaodong Yang, Ziyi Wang, Minne Li, Ying Wen,
  Joel Jennings, and Jun Wang.
\newblock Learning in nonzero-sum stochastic games with potentials.
\newblock {\em International Conference on Machine Learning (ICML) 2021}, 2021.

\bibitem{Mni15}
Volodymyr Mnih, Koray Kavukcuoglu, David Silver, Andrei~A. Rusu, Joel Veness,
  Marc~G. Bellemare, Alex Graves, Martin Riedmiller, Andreas~K. Fidjeland,
  Georg Ostrovski, Stig Petersen, Charles Beattie, Amir Sadik, Ioannis
  Antonoglou, Helen King, Dharshan Kumaran, Daan Wierstra, Shane Legg, and
  Demis Hassabis.
\newblock Human-level control through deep reinforcement learning.
\newblock {\em Nature}, 518(7540):529--533, Feb 2015.

\bibitem{moravvcik2017deepstack}
Matej Morav{\v{c}}{\'\i}k, Martin Schmid, Neil Burch, Viliam Lis{\`y}, Dustin
  Morrill, Nolan Bard, Trevor Davis, Kevin Waugh, Michael Johanson, and Michael
  Bowling.
\newblock Deepstack: Expert-level artificial intelligence in heads-up no-limit
  poker.
\newblock {\em Science}, 356(6337):508--513, 2017.

\bibitem{nash1950equilibrium}
John~F Nash et~al.
\newblock Equilibrium points in n-person games.
\newblock {\em Proceedings of the national academy of sciences}, 36(1):48--49,
  1950.

\bibitem{oroojlooyjadid2019review}
Afshin OroojlooyJadid and Davood Hajinezhad.
\newblock A review of cooperative multi-agent deep reinforcement learning.
\newblock {\em arXiv preprint arXiv:1908.03963}, 2019.

\bibitem{PPP17}
Gerasimos Palaiopanos, Ioannis Panageas, and Georgios Piliouras.
\newblock Multiplicative weights update with constant step-size in congestion
  games: Convergence, limit cycles and chaos.
\newblock In Isabelle Guyon, Ulrike von Luxburg, Samy Bengio, Hanna~M. Wallach,
  Rob Fergus, S.~V.~N. Vishwanathan, and Roman Garnett, editors, {\em Advances
  in Neural Information Processing Systems 30: Annual Conference on Neural
  Information Processing Systems 2017, December 4-9, 2017, Long Beach, CA,
  {USA}}, pages 5872--5882, 2017.

\bibitem{rosenthal73}
R.W. Rosenthal.
\newblock A class of games possessing pure-strategy {N}ash equilibria.
\newblock {\em International Journal of Game Theory}, 2(1):65--67, 1973.

\bibitem{roughgarden2009intrinsic}
Tim Roughgarden.
\newblock Intrinsic robustness of the price of anarchy.
\newblock In {\em Proceedings of the forty-first annual ACM symposium on Theory
  of computing}, pages 513--522, 2009.

\bibitem{silver2017mastering}
David Silver, Julian Schrittwieser, Karen Simonyan, Ioannis Antonoglou, Aja
  Huang, Arthur Guez, Thomas Hubert, Lucas Baker, Matthew Lai, Adrian Bolton,
  et~al.
\newblock Mastering the game of go without human knowledge.
\newblock {\em nature}, 550(7676):354--359, 2017.

\bibitem{sunehag2017value}
Peter Sunehag, Guy Lever, Audrunas Gruslys, Wojciech~Marian Czarnecki, Vinicius
  Zambaldi, Max Jaderberg, Marc Lanctot, Nicolas Sonnerat, Joel~Z Leibo, Karl
  Tuyls, et~al.
\newblock Value-decomposition networks for cooperative multi-agent learning.
\newblock {\em International Conference on Autonomous Agents and Multiagent
  Systems (AAMAS)}, 2017.

\bibitem{vinyals2019grandmaster}
Oriol Vinyals, Igor Babuschkin, Wojciech~M Czarnecki, Micha{\"e}l Mathieu,
  Andrew Dudzik, Junyoung Chung, David~H Choi, Richard Powell, Timo Ewalds,
  Petko Georgiev, et~al.
\newblock Grandmaster level in starcraft ii using multi-agent reinforcement
  learning.
\newblock {\em Nature}, 575(7782):350--354, 2019.

\bibitem{zhang2021gradient}
Runyu Zhang, Zhaolin Ren, and Na~Li.
\newblock Gradient play in multi-agent markov stochastic games: Stationary
  points and convergence.
\newblock {\em arXiv preprint arXiv:2106.00198}, 2021.

\end{thebibliography}

\appendix

\section{Proof of Claim \ref{cl:mahalanobis}}
\begin{proof}
Let $\norm{\tilde{\theta}-\theta}_{\infty} \leq \frac{\eta}{1-\gamma}$. We show that 

\begin{equation}\label{eq:init}
1-\frac{2\eta}{1-\gamma} \leq \frac{\pi_{\theta_i}(a|s)}{\pi _{\tilde{\theta}_i}(a|s)} \leq 1+\frac{4\eta}{1-\gamma}
\end{equation}

for all agents $i$ and $a \in A_i.$
It holds that 
\begin{align*}
\pi_{\theta_i}(a|s) &= \frac{e^{\theta_{i,s,a}}}{\sum_{a'\in A_i}e^{\theta_{i,s,a'}}} \\&\leq \frac{e^{\tilde{\theta}_{i,s,a}+\frac{\eta}{1-\gamma}}}{\sum_{a'\in A_i}e^{\tilde{\theta}_{i,s,a'}-\frac{\eta}{1-\gamma}}} \\&= \pi_{\tilde{\theta}_i}(a|s) e^{\frac{2\eta}{1-\gamma}}\leq \pi_{\tilde{\theta}_i}(a|s) \left(1+4\frac{\eta}{1-\gamma}\right),
\end{align*}
where the last inequality works as $e^x \leq 2x+1$ for $x\in [0,1/2],$ i.e., we assume $\eta \leq \frac{(1-\gamma)}{2}.$ 
Similarly,
\begin{align*}
\pi_{\theta_i}(a|s) &= \frac{e^{\theta_{i,s,a}}}{\sum_{a'\in A_i}e^{\theta_{i,s,a'}}} \\&\geq \frac{e^{\tilde{\theta}_{i,s,a}-\frac{\eta}{1-\gamma}}}{\sum_{a'\in A_i}e^{\tilde{\theta}_{i,s,a'}+\frac{\eta}{1-\gamma}}} \\&= \pi_{\tilde{\theta}_i}(a|s) e^{-\frac{2\eta}{1-\gamma}}\geq \pi_{\tilde{\theta}_i}(a|s) \left(1-2\frac{\eta}{1-\gamma}\right),
\end{align*}
where the last inequality comes from the fact that $e^x \geq x+1$ for all $x.$
We focus now on the Hessian of $\Phi^{\pi_{\theta}}(\mu)$ with respect to $\theta$, that is $\nabla^2 \Phi^{\pi_{\theta}}(\mu).$ It suffices to show that the spectral norm of the matrix $\left(D^{\theta}\right) ^{-\frac{1}{2}}\cdot \nabla^2_{\theta} \Phi^{\pi_{\tilde{\theta}}}(\mu)\cdot \left(D^{\theta}\right) ^{-\frac{1}{2}}$ \footnote{which has the same spectral norm with $\left(D^{\theta}\right) ^{-1}\cdot \nabla^2_{\theta} \Phi^{\pi_{\tilde{\theta}}}(\mu)$} is bounded by $L$ (notice that the Hessian is computed at $\tilde{\theta}$ and the diagonal matrix $D^{\theta}$ is computed at $\theta$ with $\norm{\tilde{\theta}-\theta}_{\infty} \leq \frac{\eta}{1-\gamma}$). 

Fix agent $i$, scalar $t\geq 0$, and a vector $u = \left(D^{\tilde{\theta}}\right)^{-1}v$ with $v$ a unit vector. Moreover, let $V(t) = V^{\pi_{\theta+tu}}_i(\mu).$ 

It holds that $$V(t) = \sum_{s\in S} \mu(s) \sum_{\a \in A} \pi_{\theta+tu}(\a|s)  Q_i^{\pi_{\theta+tu}}(s,\a).$$ Hence taking the second derivative we have
\begin{equation}\label{eq:VQ}
\begin{split}
\frac{d^2 V(0)}{dt^2} = &\sum_{s}\frac{\mu(s)}{d_{\mu}^{\pi_{\tilde{\theta}}}(s)}\sum_{\a \in A} \pi_{\theta}(\a|s) \frac{d^2 Q_i^{\pi_{\theta+tu}}(s,\a)}{dt^2}\Bigr|_{t=0}
\\&+\sum_{s}\frac{\mu(s)}{d_{\mu}^{\pi_{\tilde{\theta}}}(s)}\sum_{\a \in A} \frac{d^2\pi_{\theta+tu}(\a|s)}{dt^2}\Bigr|_{t=0}  Q_i^{\pi_{\theta}}(s,\a)
\\&+ 2\sum_{s\in S}\frac{\mu(s)}{d_{\mu}^{\pi_{\tilde{\theta}}}(s)}\sum_{\a \in A}  \frac{d\pi_{\theta+tu}(\a|s)}{dt}\Bigr|_{t=0}\frac{d Q_i^{\pi_{\theta+tu}}(s,\a)}{dt}\Bigr|_{t=0}.
\end{split}
\end{equation}

The following calculations hold (see \cite{Aga20}, section D).
\begin{enumerate}
    \item Let $a_i$ be the $i$-th coordinate of $\a$ and $u_{i,s,b_i}$ the coordinate of $u$ corresponding to agent $i$ choosing action $b_i$ at state $s.$ We have\\ $\frac{d\pi_{\theta+tu}(\a|s)}{dt}\Bigr|_{t=0} = \pi_{\theta}(\a|s)\sum_{i\in N}\sum_{b_i \in A_i} u_{i,s,b_i} \cdot  (\textbf{1}_{b_i=a_i} - \pi_{\theta_i}(b_i|s)).$ We conclude that \begin{align*}\left|\frac{d\pi_{\theta+tu}(\a|s)}{dt}\Bigr|_{t=0}\right| &\leq \left|\sum_{i\in N}\sum_{b_i\in A_i} \frac{\pi_{\theta}(\a|s)\cdot  (\textbf{1}_{b_i=a_i} - \pi_{\theta_i}(b_i|s))}{\pi_{\tilde{\theta}_i}(b_i|s)}\right| \\&\stackrel{(\ref{eq:init})}{\leq} \sum_{i\in N} \left(1+\frac{4\eta}{1-\gamma}\right)\pi_{\theta_{-i}}(\a_{-i}|s)(1-\pi_{\theta_i}(a_i|s))\\&+\left(1-\frac{2\eta}{1-\gamma}\right)\sum_{i \in N}\sum_{b_i\neq a_i}\pi_{\theta}(\a|s) \leq \left(1+\frac{4\eta}{1-\gamma}\right)n A_{\max}.
    \end{align*}
    \item We also have $\frac{d^2\pi_{\theta+tu}(\a|s)}{dt^2}\Bigr|_{t=0} = \pi_{\theta}(\a|s)\sum_{i,j \in N}\sum_{b_i \in A_i}\sum_{b_j \in A_j}u_{i,s,b_i} u_{j,s,b_j} (\textbf{1}_{b_i=a_i} - \pi_{\theta_i}(b_i|s)) \cdot (\textbf{1}_{b_j=a_j} - \pi_{\theta_j}(b_j|s)) - \pi_{\theta}(\a|s) \sum_{i\in N}\sum_{b_i,c_i\in A_i} u_{i,s,b_i}u_{i,s,c_i}\pi_{\theta_i}(b_i|s)) (\textbf{1}_{b_i=c_i} - \pi_{\theta_i}(c_i|s)).$
    Similarly, We conclude that \begin{align*}\left|\frac{d^2\pi_{\theta+tu}(\a|s)}{dt^2}\Bigr|_{t=0}\right| &\leq \left(1+\frac{4\eta}{1-\gamma}\right)^2 n^2 A^2_{\max}
    \end{align*}
\end{enumerate}

To bound the derivative of the $Q$-function, observe that
$Q_i^{\pi_{\theta+tu}}(s,\a) = e^{\top}_{s,\a}(I - \gamma P(t))^{-1}r$, where $r(s,\a)$ is the expected reward of agent $i$  if agent choose action $\a$ at state $s$ and $P(t)$ is state-action transition matrix  w.r.t the joint distribution of all agents and the environment. 

It was shown in \cite{Aga20,leonardos21} that
\begin{equation}\label{eq:firstQ}
\begin{split}
    \left|\frac{d Q_i^{\pi_{\theta}+tu}(s,\a)}{dt}\right| &= \gamma \left|e^{\top}_{s,a}(I - \gamma P(0))^{-1}\frac{dP(0)}{dt}(I - \gamma P(0))^{-1}r\right|
    \leq \frac{\gamma\sqrt{|A_{i}|}}{(1-\gamma)^2} \leq \frac{\gamma \sqrt{A_{\max}}}{(1-\gamma)^2},
\end{split}
\end{equation}
and also
\begin{equation}\label{eq:secondQ}
\begin{split}
    \left|\frac{d^2 Q_i^{\pi_{\theta}+tu}(s,\a)}{dt^2}\right| &= 2\gamma^2 \left|e^{\top}_{s_0,a}(I - \gamma P(0))^{-1}\frac{dP(0)}{dt}(I - \gamma P(0))^{-1}\frac{dP(0)}{dt}(I - \gamma P(0))^{-1}r\right|
    \\& \leq \frac{2\gamma^2|A_{i}|}{(1-\gamma)^3} \leq \frac{2\gamma A_{\max} }{(1-\gamma)^3},
\end{split}
\end{equation}
Combining the above and assuming $\eta \leq \frac{1-\gamma}{2}$ we get
\[
\left|\frac{d^2 V(0)}{dt^2}\right| = \sum_{s} \mu(s)\left( \frac{6\gamma A_{\max}}{(1-\gamma)^3} + 9n^2 A^2_{\max} + \frac{3nA_{\max}\gamma\sqrt{A_{\max}}}{(1-\gamma)^2}\right) \leq M \frac{27n^2 A^2_{\max}}{(1-\gamma)^3},
\]
where $M$ is the mismatch coefficient. We choose $L = \frac{27n^2 A^2_{\max}M}{(1-\gamma)^3}.$ Since $u$ is arbitrary, the Claim is proved.

\end{proof}

\end{document}